\documentclass[letterpaper,10pt,conference]{formatting/ieeeconf}
\IEEEoverridecommandlockouts
\let\labelindent\relax
\usepackage{amsfonts}       

\usepackage{amsthm}
\usepackage{mathtools}      
\usepackage{amssymb}        
\usepackage{filecontents}
\usepackage{graphicx}       
\usepackage{marginnote}     
\usepackage{marvosym}       
\usepackage{overpic}        
\usepackage{tabularx}
\usepackage{cite}
\usepackage{color}
\usepackage[linesnumbered,algoruled,boxed,lined,noend]{algorithm2e}
\usepackage[font=small,labelfont=bf]{caption}

\usepackage{float} 
\usepackage{wrapfig}
\usepackage{comment}

\let\labelindent\relax
\usepackage{enumitem}
\setlist{leftmargin=3.5mm}

\usepackage{tikz} 
\usepackage[a-2b,mathxmp]{pdfx}[2018/12/22]

\usepackage{caption}
\usepackage{subcaption}

\usepackage[normalem]{ulem}
\usepackage{epstopdf}
\usepackage{enumitem}
\usepackage[normalem]{ulem}
\usepackage{lipsum}
\usepackage[a-2b,mathxmp]{pdfx}[2018/12/22]

\newcommand{\argmin}{\mathop{\mathrm{arg\,min}}}

\def\RVG{\mathcal{RVG}\xspace}

\usepackage{xcolor}
\usepackage{xargs} 
\newif\ifdraft
\draftfalse

\ifdraft
\usepackage[paperheight=11in,paperwidth=9.5in,
			left=1.25in,right=1.25in,
			top=0.75in,bottom=0.75in,
			heightrounded,marginparwidth=1.2in,
			marginparsep=0.05in]{geometry}
\usepackage{xcolor}
\usepackage{xargs} 
\usepackage[textsize=footnotesize]{todonotes}
\newcommandx{\nt}[2][1=]{\todo[linecolor=red,
			backgroundcolor=red!10,bordercolor=red,#1]{#2}}
\newcommandx{\jy}[2][1=]{\todo[linecolor=green,
			backgroundcolor=green!10,bordercolor=green,#1]{JY:#2}}
\newcommandx{\dz}[2][1=]{\todo[linecolor=red,
			backgroundcolor=red!10,bordercolor=red,#1]{DZ:#2}}
\else
\newcommand{\nt}[1]{{}}
\newcommand{\jy}[1]{{}}
\newcommand{\dz}[1]{{}}
\fi

\newif\iftwocolumn
\twocolumntrue

\newtheorem{lemma}{Lemma}[section]

\newtheorem{corollary}{Corollary}[section]
\newtheorem{theorem}{Theorem}[section]
\theoremstyle{definition}
\newtheorem{definition}{Definition}[section]
\theoremstyle{remark}

\SetAlgoSkip{}
\SetKwProg{Fn}{Function}{}{}
\SetKwComment{Comment}{$\triangleright$\ }{}

\makeatletter
\def\subsubsection{\@startsection{subsubsection}
                                 {3}
                                 {\z@ \hspace*{1mm}}
                                 {0ex plus 0.1ex minus 0.1ex}
                                 {0ex}
                                 {\normalfont\normalsize\itshape}}
\def\section{\@startsection{section}{1}{\z@}{1.5ex}{0.7ex}%
{\normalfont\normalsize\centering\scshape}}
\makeatother
\usepackage{xspace}
\newtheorem{assumption}{Assumption}[section]
\newcommand{\subfloat}[2][]{\subcaptionbox{#1}{#2}}

\newcommand{\dRVG}{\mbox{\upshape dRVG}\xspace}
\renewcommand{\RVG}{\textsc{RVG}\xspace}
\newcommand{\W}{\mathcal W}
\newcommand{\Cfree}{\mathcal C_{\mathrm{free}}}
\newcommand{\Gfull}{\mathcal G_{\mathrm{full}}}

\title{\LARGE \dRVG: Quadtree-Guided, Resolution-Complete Online Motion Planning for Polygonal Robots in Unknown Environments}

\author{Duo Zhang \qquad Hechen Zhang \qquad Junshan Huang \qquad Jingjin Yu%
\thanks{D. Zhang, H. Zhang, J. Huang, and J. Yu are with the Department of Computer Science, Rutgers, the State University of New Jersey, Piscataway, NJ, USA. E-Mails: \texttt{\{duo.zhang, hechen.zhang, junshan.huang, jingjin.yu\}@rutgers.edu}.}%
}

\begin{document}
\maketitle
\raggedbottom
\AddToHookNext{shipout/after}{\begingroup\globaldefs=1\flushbottom\endgroup}
\suppressfloats[t]
\thispagestyle{empty}
\pagestyle{empty}

\begin{abstract}
We present the dynamic rotation-stacked visibility graph (\dRVG), an online motion planner that guides polygonal robots to specified goals in initially unknown, static environments. It merges local roadmaps from successive observations to plan collision-free translations and rotations without a uniform position grid. A spatial quadtree schedules sensing goals across regions to reduce repeated visits while retaining all orientation configurations for routing. Under exact sensing and geometric computation and star-shaped robot and envelope assumptions, \dRVG with center scans is resolution-complete relative to full-map RVG at the same angular resolution. In experiments using footprint scans, \dRVG solves all 140 cases across 20 difficult maps and seven angular resolutions within a 20\,s planning budget, with a median planning time of 1.18\,s at 360 orientation layers. Six microMVP demonstrations illustrate the complete online planning loop on a physical robot.
\end{abstract}

\section{Introduction}\label{sec:intro}
\begin{figure}[t]
  \centering
  \includegraphics[width=1\columnwidth]{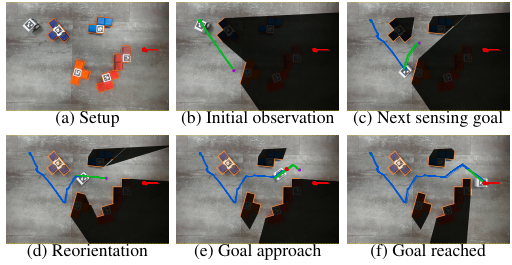}
  \caption{Online \dRVG execution on microMVP: (a) setup; (b--f) navigation. Temporary sensing goals (magenta) guide translation and reorientation until the robot reaches its goal. Blue is the executed trajectory, green the planned path, and black the space not yet observed.}
  \label{fig:real-robot}
\end{figure}

Autonomous navigation in unfamiliar environments enables robotic inspection and delivery before a complete map is available. Deployments at Fukushima demonstrated the value of accessing areas too dangerous for human inspection~\cite{nagatani2013emergency}. Indoor service robots may likewise need to reach a specified room while discovering the surrounding layout~\cite{stein2018subgoals}. These applications require reaching a known destination while discovering obstacles and a feasible route through sensing and motion. Robot shape makes this task more demanding: an elongated or nonconvex body may need to rotate to pass between furniture or debris. We develop an online motion planner that guides a robot with a polygonal footprint to a specified position and orientation among initially unknown, static obstacles. It selects collision-free motions and sensing locations to reveal routes toward the goal. Figure~\ref{fig:real-robot} shows goal arrival before full workspace observation.

Our method, the \emph{dynamic rotation-stacked visibility graph} (\dRVG), builds a graph of feasible motions as observations arrive. It extends the full-map rotation-stacked visibility graph (\RVG)~\cite{zhang2025rvg}, which represents translation and rotation across discrete orientation intervals without imposing a uniform position grid. At each sensing location, \dRVG constructs a local graph inside observed free space and merges it with earlier graphs. When the final goal is not yet reachable, the robot follows a known-free route to a temporary sensing goal. A spatial quadtree groups candidate sensing locations into regions and defers a region after it supplies a target, reducing repeated visits to nearby locations. All orientation configurations remain available for motion planning.

Under exact sensing and the stated star-shaped robot and envelope assumptions, we prove that \dRVG with center scans reaches the goal whenever full-map RVG at the same angular resolution contains a path. Quadtree scheduling preserves this guarantee while eventually serving every persistent reachable candidate. The proof uses center scans to establish a finite set of candidate sensing sites. Experiments use scans from the robot's center and footprint vertices to support more general shapes. A non-star-shaped counterexample shows that this guarantee does not extend to all polygonal robots.

On 20 difficult maps at seven angular resolutions, \dRVG solves all 140 cases within a 20\,s planning budget. At 360 orientation layers, its median planning time is 1.18\,s, and it solves 20/20 maps; repeated A*~\cite{hart1968astar}, D* Lite~\cite{koenig2002dstarlite}, $\mathrm{RRT}^{X}$~\cite{otte2016rrtx}, and EIT*~\cite{strub2022eitstar} solve at most 18, 10, 1, and 5 maps, respectively, across their tested grid or collision-check resolutions under the protocol in Section~\ref{sec:evaluation}. Quadtree scheduling improves success from 93/140 for naive greedy selection to 140/140; on jointly solved cases, greedy requires $5.51\times$ the planning time in geometric mean. Hardware demonstrations further illustrate the complete online planning loop on a physical robot.

We make the following contributions:
\begin{itemize}[topsep=2pt,itemsep=1pt,parsep=0pt,beginpenalty=10000]
  \item We develop an incremental rotation-stacked visibility graph that integrates local observations into a persistent roadmap for collision-free polygonal-robot motion in $SE(2)$.
  \item We design a quadtree-guided sensing policy that reduces repeated visits to nearby sites while retaining all orientation configurations for motion planning.
  \item We prove fixed-resolution completeness relative to full-map RVG for center scans under exact sensing and star-shaped robot and envelope assumptions, and establish its geometric limits through a non-star-shaped counterexample.
  \item We evaluate \dRVG against lattice and sampling-based planners, isolate the benefit of spatial scheduling through ablation, and demonstrate the online planning loop in hardware.
\end{itemize}

\section{Related Work}\label{sec:related-work}
\paragraph{Geometric Planning and Replanning}
Configuration-space and visibility-graph methods encode robot geometry~\cite{lozano1983cspace,lozano1979collision}. Classical treatments address rigid polygonal bodies among polygonal barriers~\cite{schwartz1983piano} and disc motion through retraction methods~\cite{odunlaing1985retraction}; state lattices and RVG capture orientation-dependent feasibility~\cite{pivtoraiko2009lattices,zhang2025rvg}. Alternatives include any-angle search~\cite{daniel2010theta}, sampling-based planning~\cite{kavraki1996prm,lavalle1998rrt,karaman2011sampling}, and potential-based control~\cite{khatib1986potential,rimon1992navigation}. As obstacles are discovered, repeated A*~\cite{hart1968astar} recomputes routes, whereas D*~\cite{stentz1994dstar}, LPA*~\cite{koenig2004lpa}, D* Lite~\cite{koenig2002dstarlite}, AD*~\cite{likhachev2005adstar}, and lazy incremental search~\cite{lim2024lazy} reuse search information. Dynamic RRT~\cite{ferguson2006replanning} and $\mathrm{RRT}^{X}$~\cite{otte2016rrtx} repair sampling-based structures; fast planners such as EIT* also support replanning from scratch~\cite{strub2022eitstar,sabbadini2026replanning}. These planning techniques address how to find or update a route in the current map. Navigation through unknown space also requires choosing where to observe when that map does not yet contain a route to the goal. Our approach couples incremental geometric planning with this sensing decision.

\paragraph{Sensor-Based Roadmaps}
Visibility, Voronoi, and gap representations support navigation as geometry is revealed~\cite{oommen1987learned,rao1995navigation,tovar2005gap,choset2000incremental}. Lee and Choset construct sensor-based roadmaps for convex and multi-convex rigid bodies in $SE(2)$~\cite{lee2005convex,lee2005multiconvex}. SP$^2$ATM combines topological mapping with intermediate-goal selection~\cite{ge2011sp2atm}, and FAR Planner updates visibility graphs online~\cite{yang2022far}. Chen et al. select short-term goals and address dead ends through a starshaped roadmap for a disk-shaped robot~\cite{chen2025starshaped}. SE(2) NavMesh builds heading-dependent traversable regions from observations~\cite{shi2026se2navmesh}. We focus on preserving connections between successive local roadmaps for a polygonal robot whose feasible motions depend on orientation. \dRVG combines online rotation-layered visibility graphs with spatial sensing scheduling. Its completeness guarantee is relative to full-map RVG at the same angular resolution, under the stated center-scan and geometric assumptions.

\paragraph{Sensing for Goal Navigation}
Bug algorithms and TangentBug pursue goals through local sensing and boundary following~\cite{lumelsky1987path,kamon1998tangentbug}, while frontier-based planners expand observed space~\cite{yamauchi1997frontier,senarathne2015frontier,zhou2021fuel}. Janson et al. formulate safe policies and information-aware optimality benchmarks~\cite{janson2018safe}. Learning over Subgoals predicts the consequences of exploratory actions~\cite{stein2018subgoals}; subsequent work learns their information value~\cite{arnob2024information}. UPEN and CogniPlan predict unseen occupancy or layouts to guide exploration and point-goal navigation~\cite{georgakis2022uncertainty,wang2025cogniplan}. The sensing objective in our setting is to reveal a feasible route to a specified goal; complete workspace coverage is unnecessary. We use observed geometry to generate reachable sensing candidates, without a learned model of unseen space. \dRVG schedules these sites across quadtree regions to reduce repeated regional visits while retaining all orientation configurations for routing. This separation allows the sensing policy to defer a region without removing configurations needed to pass through it. Under the analysis assumptions, every persistent reachable candidate is eventually served.

\section{Preliminaries}\label{sec:problem}
Let $\W\subset\mathbb R^2$ be a compact polygonal workspace with finitely many initially unknown, static, pairwise interior-disjoint polygonal obstacles, whose union is $O$. A holonomic robot has a compact simple polygonal footprint $R$ and reference and rotation point $c\in R$. At pose $q=(p,\theta)\in SE(2)$, with $p=(x,y)$ and planar rotation $\mathsf R_\theta$, its footprint is $R(q)=p+\mathsf R_\theta(R-c)$. The collision-free space is
\begin{equation}
\Cfree=\{q:R(q)\subseteq\W,\ \operatorname{int}R(q)\cap O=\emptyset\}.
\label{eq:configuration-free-space}
\end{equation}
Boundary contact is allowed. Given $q_s,q_g\in\Cfree$, we seek a path $\tau:[0,1]\to\Cfree$ from $q_s$ to $q_g$ with low executed cost
\begin{equation}
J(\tau)=\alpha\int_0^1\sqrt{dx^2+dy^2}
+\beta\int_0^1|d\theta|,
\qquad \alpha,\beta>0,
\label{eq:path-cost}
\end{equation}
which weights translation length and cumulative rotation.

\subsection{Rotation-Stacked Visibility Graphs}
RVG~\cite{zhang2025rvg} uses $N$ orientation layers with intervals and representative angles ($k=0,\ldots,N-1$)
\begin{equation}
I_k=\left[\frac{2\pi k}{N},\frac{2\pi(k+1)}{N}\right],
\qquad \hat\theta_k=\frac{(2k+1)\pi}{N},
\label{eq:orientation-layers}
\end{equation}
modulo $2\pi$. A conservative polygonal envelope $\mathcal B_k$ contains the robot throughout $I_k$:
\begin{equation}
\bigcup_{\theta\in I_k}\mathsf R_\theta(R-c)
\subseteq\mathcal B_k.
\label{eq:layer-envelope}
\end{equation}
As $N\to\infty$, the envelope converges to the exact swept region over $I_k$.

In each layer, RVG grows obstacles by a Minkowski construction using $\mathcal B_k$ and builds a reduced visibility graph in the free translation domain at angle $\hat\theta_k$. Within-layer edges encode translations; same-position links between adjacent layers encode collision-free rotations. Edge weights follow Eq.~\eqref{eq:path-cost}. Let $\Gfull=(V_N,E_N)$ be the full-map RVG after connecting $q_s$ and $q_g$ by collision-free query edges.

\subsection{Observation Model}
Let $\operatorname{Vis}(z)$ denote closed free-space visibility from $z$, with workspace-wide range and no robot self-occlusion. Center and vertex scans are
\begin{equation}
S_c(q)=\operatorname{Vis}(p),\qquad
S_V(q)=\bigcup_{z\in\operatorname{vert}(R(q))}\operatorname{Vis}(z).
\label{eq:footprint-observation}
\end{equation}
The footprint scan is $S(q)=S_c(q)\cup S_V(q)$. At a fixed position, only vertex scans depend on orientation. Unobserved obstacle geometry remains unknown.

At a collision-free pose, $R(q)\subseteq S_V(q)\subseteq S(q)$: triangulating $R(q)$ places every footprint point in a triangle of robot vertices, from which it is visible. In contrast, an obstacle can hide part of a nonconvex footprint from a center scan, as illustrated by the U-shaped robot in Fig.~\ref{fig:footprint-sensing}.

\section{Dynamic RVG}\label{sec:algorithm}
\begin{figure}[t]
  \centering
  \subfloat[Select a sensing goal]{\includegraphics[angle=90,width=0.31\columnwidth]{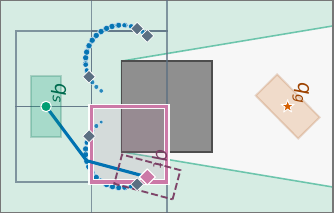}}\hfill
  \subfloat[Add RVG sites]{\includegraphics[angle=90,width=0.31\columnwidth]{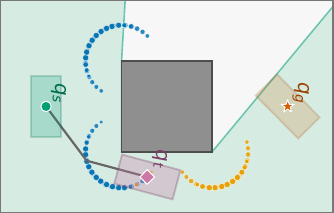}}\hfill
  \subfloat[Select next goal]{\includegraphics[angle=90,width=0.31\columnwidth]{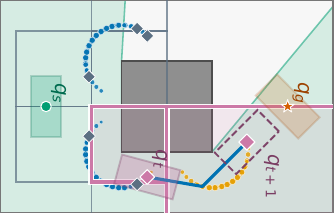}}
  \caption{Two consecutive sensing-goal selections ($N=36$). (a) The quadtree built from the start observation selects $q_t$ (purple diamond), with its known-free route from $q_s$ shown in blue. (b) After reaching $q_t$, the robot scans and builds a local RVG patch: blue dots are existing sites and orange dots are newly added sites. (c) The quadtree incorporates both sets before selecting $q_{t+1}$; the region served at $q_t$ remains deferred during the current epoch. Diamonds denote leaf representatives; purple highlights selected representatives and served regions. The blue curve in (c) connects the two sensing goals.}
  \label{fig:overview}
\end{figure}

\begin{figure}[t]
  \centering
  \subfloat[Single origin]{\includegraphics[width=0.35\columnwidth]{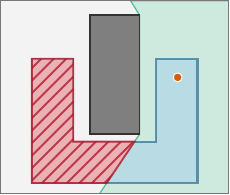}}\hspace{10mm}
  \subfloat[Footprint scan]{\includegraphics[width=0.35\columnwidth]{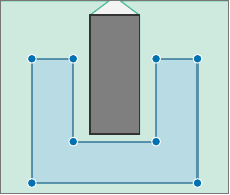}}
  \caption{For a robot that is not star-shaped about its sensing center, a center scan can leave part of the footprint unobserved. (a) The obstacle in the U-shaped robot's concavity hides the hatched part of the footprint from a single origin (orange). (b) Scans from the footprint vertices (blue) collectively cover both arms and the connecting body, providing the full-footprint coverage missing from the single-origin scan. Green denotes observed free space; gray denotes obstacles.}
  \label{fig:footprint-sensing}
\end{figure}

\subsection{Sense--Construct--Merge--Move}
We maintain observed free space $\mathcal K$, roadmap $\mathcal G$, scanned sites $E$, and sensing quadtree $\mathcal T$.

\emph{Sense.} At pose $q$, initially $q_s$, the robot takes a center scan $S_c(q)$ or footprint scan $S(q)$. We add the observation $S$ to $\mathcal K$ and the current site to $E$.

\emph{Construct.} We build $\mathcal G^{\rm loc}$ in the scan domain $S$ at fixed $N$. Layer envelopes grow obstacles by Minkowski sums; visibility edges and same-position links between adjacent layers encode collision-free translations and rotations. We insert $q$ and, when feasible, $q_g$ through collision-free query edges. Goal-position visibility alone does not establish configuration feasibility.

\emph{Merge.} We merge $\mathcal G^{\rm loc}$ into the accumulated roadmap $\mathcal G$, retaining existing vertices and edges as well as local components that are not yet reachable. Under Assumption~\ref{ass:star-exact}, a reached temporary goal is shared by the accumulated and new local graphs: its layer envelope lies in both observed domains (Lemma~\ref{lem:anchored-merge}). Its local component therefore joins the accumulated start component. Routes combine edges from multiple scans; static obstacles keep earlier edges valid.

\emph{Move.} A shortest-path search checks whether $q_g$ is reachable in $\mathcal G$. If so, the robot follows that route and finishes. Otherwise, the quadtree selects a reachable configuration at an unscanned site as a temporary sensing goal. The robot executes its route and senses again on arrival. Algorithm~\ref{alg:drvg} repeats these four steps until the goal is reached, no reachable unscanned site remains, or the planning-time budget expires. The last two outcomes return failure and timeout, respectively; the completeness analysis concerns the loop without a time limit.

\subsection{Quadtree-Based Selection of Temporary Sensing Goals}
RVG can place multiple orientation configurations at one position. Selecting sensing goals independently by configuration may repeatedly favor the same region.

We first group configurations at the same position into a \emph{spatial site} $s$. The list $\mathcal Q(s)$ keeps all of its orientation configurations, and $\mathrm{site}(v)$ gives the site of configuration $v$. Once the robot reaches and scans one configuration, its site enters $E$ and is no longer a sensing target. Its configurations remain in $\mathcal G$, so later routes can still pass through that position at different orientations. For center scans, all configurations at one position share the same observation. For the footprint-scan experiments, observations may vary with orientation, so scanning each site once is a scheduling choice.

We store sites in a quadtree $\mathcal T$ over $\W$. To prevent new sites from immediately drawing the robot back to a recently scanned region, we schedule leaves in rounds called \emph{epochs}. A leaf is eligible when it contains a reachable, unscanned site. After a leaf supplies a sensing goal, it is marked served and set aside while other eligible leaves remain unserved. When every eligible leaf has been served, a new epoch makes them available again. Scanned sites remain in $E$ across epochs; revisiting a region requires a new site. Exploration stops when none is reachable.

The default policy scores each reachable candidate $v$ by its position and orientation relative to the goal:
\begin{equation}
h(v)=\alpha\|p_v-p_g\|_2+
\beta\,|\operatorname{wrap}(\theta_v-\theta_g)|,
\label{eq:selector-score}
\end{equation}
where $p_v,p_g$ and $\theta_v,\theta_g$ are the candidate and goal positions and orientations, $\alpha,\beta$ are the cost weights in Eq.~\eqref{eq:path-cost}, and $\operatorname{wrap}$ gives the angular difference in $[-\pi,\pi]$. Each available leaf nominates its lowest-score reachable configuration, considering all orientations at its unscanned sites; the best nominee becomes the temporary goal. The score ranks sensing goals; execution follows a shortest path in $\mathcal G$. The $g+h$ variant also accounts for this route cost.

For Algorithm~\ref{alg:quadtree}, Dijkstra search from $q$ provides graph distances $d(v)$ and predecessors $\pi$, with $d(v)=\infty$ for unreachable configurations; an absent goal vertex also has infinite distance. Let $\mathcal C$ contain the unscanned sites, excluding the current site.

\begin{samepage}
Let $\ell$ denote the square region of a quadtree leaf. The reachable candidates in this region are
\begin{equation}
\mathcal A_\ell=\{v\in\mathcal Q(s):s\in\mathcal C,\ p_v\in\ell,\ d(v)<\infty\}.
\label{eq:leaf-candidates}
\end{equation}
\end{samepage}
The set $\mathcal L_{\rm elig}$ collects leaves with nonempty $\mathcal A_\ell$, and $\mathcal L$ keeps those not yet served in the current epoch. For the $g+h$ variant, both nomination steps use $d(v)+h(v)$.

If the goal is not yet reachable, we update the quadtree with the current unscanned sites and their orientation configurations, retaining served status across updates. When a leaf splits, the child containing its previously selected sensing site inherits its served status. The algorithms denote the selected configuration by $q'$. Figure~\ref{fig:overview} shows successive goals $q_t$ and $q_{t+1}$, with the region served at $q_t$ deferred after incorporating the new scan.

\begin{algorithm}[t]
\small
$q\leftarrow q_s$; $\mathcal K,\mathcal G,E\leftarrow\emptyset$;
$\mathcal T\leftarrow\mathrm{Quadtree}(\W)$\;
\While{$q\ne q_g$ and planning time remains}{
$S\leftarrow\mathrm{Sense}(q)$; $\mathcal K\leftarrow\mathcal K\cup S$\;
$E\leftarrow E\cup\{\mathrm{site}(q)\}$\;
$\mathcal G^{\rm loc}\leftarrow\mathrm{RVG}(S,N;q,q_g)$\;
$\mathcal G\leftarrow\mathrm{Merge}(\mathcal G,\mathcal G^{\rm loc})$\;
$(q',\pi)\leftarrow\textsc{QuadtreeTarget}(\mathcal G,\mathcal T,q,q_g,E)$\;
\lIf{planning budget exhausted}{\Return timeout}
\If{$q'=\bot$}{\Return failure\;}
$q\leftarrow\mathrm{Execute}(\mathrm{RecoverPath}(\pi,q'))$\;
}
\lIf{$q=q_g$}{\Return the executed path}
\Return timeout\;
\caption{Online \dRVG($q_s,q_g,N$)}\label{alg:drvg}
\end{algorithm}

\begin{algorithm}[t]
\small
$(d,\pi)\leftarrow\mathrm{Dijkstra}(\mathcal G,q)$\;
\If{$d(q_g)<\infty$}{\Return $(q_g,\pi)$\;}
$\mathcal C\leftarrow\mathrm{UnscannedSites}(\mathcal G,E\cup\{\mathrm{site}(q)\})$\;
$\mathcal T.\mathrm{UpdateSites}(\mathcal C)$\;
Compute $\mathcal A_\ell$ for each leaf using Eq.~\eqref{eq:leaf-candidates}\;
$\mathcal L_{\rm elig}\leftarrow\{\ell\in\mathrm{Leaves}(\mathcal T):\mathcal A_\ell\ne\emptyset\}$\;
$\mathcal L\leftarrow\{\ell\in\mathcal L_{\rm elig}:\ell\text{ is unserved}\}$\;
\If{$\mathcal L=\emptyset$ and $\mathcal L_{\rm elig}\ne\emptyset$}{
$\mathcal T.\mathrm{NewEpoch}()$; $\mathcal L\leftarrow\mathcal L_{\rm elig}$\;
}
\If{$\mathcal L=\emptyset$}{\Return $(\bot,\pi)$\;}
\ForEach{$\ell\in\mathcal L$}{$r_\ell\leftarrow\argmin_{v\in\mathcal A_\ell}h(v)$\;}
$q'\leftarrow\argmin_{v\in\{r_\ell:\ell\in\mathcal L\}}h(v)$\;
$\mathcal T.\mathrm{MarkServed}(\operatorname{leaf}(q'))$\;
\Return $(q',\pi)$\;
\caption{\textsc{QuadtreeTarget}($\mathcal G,\mathcal T,q,q_g,E$)}\label{alg:quadtree}
\end{algorithm}

\section{Resolution-Completeness Analysis}\label{sec:analysis}
We analyze \dRVG with center scans, fixed $N$, and no time limit. All lemmas assume Assumption~\ref{ass:star-exact}; local roadmaps use $S_c(q)$.

\begin{figure}[t]
  \centering
  \includegraphics[width=0.7\columnwidth]{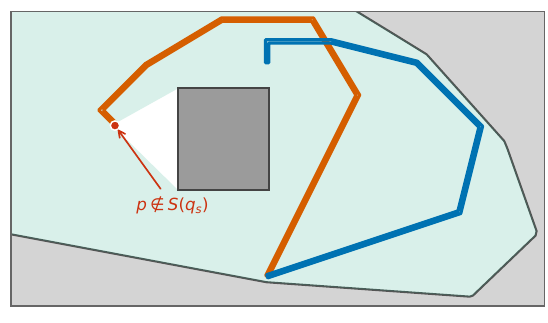}
  \caption{Geometric limit at $N=8$. Static RVG can rotate the non-star-shaped hook from the blue start footprint to the orange goal footprint. The initial footprint scan leaves the marked goal-footprint point unobserved and offers no temporary goal. Green is observed free space. This example lies outside Assumption~\ref{ass:star-exact}.}
  \label{fig:completeness-counterexample}
\end{figure}

\begin{samepage}
\begin{definition}[Relative resolution completeness]
For fixed $N$ and no time limit, a planner is resolution-complete relative to RVG if it reaches $q_g$ in finitely many steps whenever $\Gfull$ contains a $q_s$--$q_g$ path.
\end{definition}
\end{samepage}
Let $\Gfull^s=(V_N^s,E_N^s)$ be the connected component of $\Gfull$ containing $q_s$. A full-map vertex is \emph{represented} when its configuration is reachable from $q_s$ in the accumulated graph. A full-map edge is \emph{represented} by an accumulated-graph path between its endpoints; storing this path alone does not imply reachability from $q_s$.

\begin{assumption}[Star-shaped exact setting]\label{ass:star-exact}
The robot is star-shaped about $c$: $\overline{cr}\subseteq R$ for every $r\in R$. Every layer envelope $\mathcal B_k$ is star-shaped about the origin. Sensing, geometry, collision tests, RVG construction, and merging are exact in the bounded static polygonal workspace.
\end{assumption}

Local and full-map RVGs use the same orientation layers, conservative envelopes, visibility construction, and vertex propagation~\cite{zhang2025rvg}. Thus, a full-map edge has a corresponding local path whenever its entire conservative sweep is contained in the observed domain. The center observation is shared by all orientations at the scanned position, so the local graph can include connections for configurations not yet reachable from $q_s$. Merge retains these disconnected components for later use; each site is scanned at most once.

\begin{lemma}[Neighbor Reachability after Sensing]\label{lem:neighbor-revelation}
Let $u\in V_N^s$ be represented, and suppose its spatial site has been scanned. Every full-map neighbor $w$ of $u$ is then represented and connected to $u$.
\end{lemma}
\begin{proof}
A layer-$k$ translation from position $p$ by displacement $d$ sweeps points $z=p+b+\lambda d$, with $b\in\mathcal B_k$ and $\lambda\in[0,1]$. For $\mu\in[0,1]$,
\begin{equation}
p+\mu(z-p)=p+\mu b+(\mu\lambda)d \in\bigcup_{t\in[0,1]}\left(p+\mathcal B_k+t d\right),
\label{eq:star-translation-certificate}
\end{equation}
because $\mu b\in\mathcal B_k$ and $\mu\lambda\in[0,1]$. Thus the translation sweep is star-shaped about $p$. A same-position rotation between adjacent layers uses $(p+\mathcal B_k)\cup(p+\mathcal B_{k+1})$. Both envelopes are star-shaped about $p$, so their union is as well. For RVG's layer-based query insertion~\cite{zhang2025rvg}, a query at $p_a$ with orientation in $I_k$ connects through that layer. It can rotate to $\hat\theta_k$ at $p_a$ and translate to $p_b$; a goal attachment reverses this order. Both motions lie in $\bigcup_{t\in[0,1]}(p_a+\mathcal B_k+t(p_b-p_a))$, since the query footprint and within-layer rotation lie in $p_a+\mathcal B_k$. This sweep is star-shaped about either endpoint by the translation argument. Thus these attachments satisfy the same visibility certificate.

For a full-map edge $(u,w)$ with $u$ at $p$, the collision-free conservative sweep is therefore visible from $p$. A center scan contains this sweep at any scanning orientation, so the local graph contains a path joining $u$ and $w$.
Merge retains this path even if $u$ is initially unreachable. Once $u$ becomes reachable, so does $w$, without another scan.
\end{proof}

\begin{samepage}
\begin{lemma}[Reachability Preservation]\label{lem:anchored-merge}
Initialization establishes the start component. Each reached temporary goal joins its new local component to it. Observed free space and reachability are nondecreasing, but may remain unchanged after a scan.
\end{lemma}
\end{samepage}
\begin{proof}
If a full-map solution exists, its first query attachment has a sweep visible from $q_s$ by the preceding argument, so its query configuration is feasible in the initial scan. Now let $q$ be a reached temporary goal in layer $k$, at $p$. The incoming roadmap motion's conservative sweep contains the endpoint envelope $p+\mathcal B_k$ and lies in $\mathcal K_{\rm old}$. Since this envelope is collision-free and star-shaped about $p$, it is visible from $p$. Hence $p+\mathcal B_k\subseteq\mathcal K_{\rm old}\cap S_c(q)$, so $q$ is feasible in the new local layer. Merging its local and accumulated copies joins the new component to the start. Earlier vertices and edges remain valid and are retained, preserving reachability. A scan need not add new space or connections.
\end{proof}

\begin{lemma}[Quadtree Reordering]\label{lem:quadtree-neutral}
The quadtree changes only sensing order. Every persistent reachable unscanned site is selected after finitely many successful temporary-goal moves unless the true goal is reached first.
\end{lemma}
\begin{proof}
A center scan can create new boundary vertices where visibility rays meet the interior of straight edges. The observed free-space angle at each such intersection is at most $180^\circ$, so it is not a reflex vertex. Every reflex vertex therefore coincides with an obstacle corner or a corner of the outer workspace boundary. Convex corners of the grown complement require convex corners of both Minkowski summands; otherwise a translated input contains a segment through the corner. For fixed $N$, these finite sets of map and envelope vertices yield finitely many roadmap positions across all center scans. Propagation preserves positions; query insertion uses the fixed start and goal or previously selected sites. Each temporary-goal move scans a previously unscanned site. If an eligible leaf remains, Algorithm~\ref{alg:quadtree} selects a site, resetting the epoch when needed. Finite candidates and monotone reachability preclude indefinite deferral of persistent reachable sites.
\end{proof}

\begin{lemma}[Start-Component Representation]\label{lem:finite-component}
Unless the goal is reached first, every vertex of $\Gfull^s$ becomes reachable from $q_s$ after finitely many temporary-goal moves.
\end{lemma}
\begin{proof}
Induct on full-map edge distance from $q_s$, initially represented and scanned. Each vertex $w$ at the next level has a represented neighbor $u$.

If $u$'s site has been scanned, Lemma~\ref{lem:neighbor-revelation} makes $w$ reachable using the stored connections. Otherwise, $u$ remains a reachable sensing candidate, so Lemma~\ref{lem:quadtree-neutral} ensures that its site is scanned after finite delay. The same neighbor lemma then makes $w$ reachable. Lemma~\ref{lem:anchored-merge} keeps the vertices from earlier levels reachable. The finite number of distance levels completes the induction.
\end{proof}

\begin{theorem}[Fixed-resolution RVG-relative completeness]\label{thm:resolution-complete}
Under Assumption~\ref{ass:star-exact}, quadtree-guided \dRVG with center scans is resolution-complete relative to RVG. If the goal remains unreachable, exhausting all reachable unscanned sites certifies that $q_s$ and $q_g$ are disconnected in $\Gfull$.
\end{theorem}
\begin{proof}
A full-map solution places $q_g$ in $V_N^s$. Lemma~\ref{lem:finite-component} makes it reachable after finitely many moves, and the goal test in Algorithm~\ref{alg:quadtree} returns its collision-free route. Exhaustive failure in the presence of such a path would contradict this result. The certificate concerns fixed-resolution RVG disconnection, not continuous-space infeasibility; timeout provides no certificate.
\end{proof}

\begin{corollary}[Resolution completeness]\label{cor:resolution-complete}
If Assumption~\ref{ass:star-exact} holds under angular refinement, a feasible query satisfying static RVG's clearance and envelope-convergence conditions~\cite{zhang2025rvg} admits a finite $N$ at which \dRVG with center scans reaches $q_g$ after finitely many temporary-goal moves.
\end{corollary}
\begin{proof}
Static RVG supplies a full-map path at finite $N$; Theorem~\ref{thm:resolution-complete} then ensures goal arrival.
\end{proof}

The geometric condition includes convex robots with a reference point in the robot and star-shaped nonconvex robots with suitable interval envelopes. Figure~\ref{fig:completeness-counterexample} demonstrates a failure of relative completeness at $N=8$: full-map RVG finds a solution, whereas the footprint-scan variant of \dRVG does not. In this example, finer angular resolutions allow additional sensing poses and a solution. \dRVG supports non-star-shaped polygonal robots, including those in Figures~\ref{fig:footprint-sensing} and~\ref{fig:completeness-counterexample}, although the completeness guarantee does not apply to them. Footprint scans may also introduce boundary intersections absent from center scans; the finite-candidate argument above does not establish completeness for that sensing variant.

\section{Evaluation}\label{sec:evaluation}
\begin{figure*}[!t]
  \centering
  \subfloat[]{\includegraphics[width=\dimexpr\textwidth/10-1pt\relax]{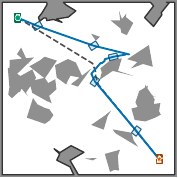}}\hfill
  \subfloat[]{\includegraphics[width=\dimexpr\textwidth/10-1pt\relax]{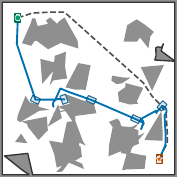}}\hfill
  \subfloat[]{\includegraphics[width=\dimexpr\textwidth/10-1pt\relax]{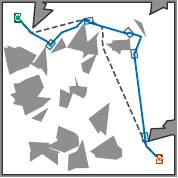}}\hfill
  \subfloat[]{\includegraphics[width=\dimexpr\textwidth/10-1pt\relax]{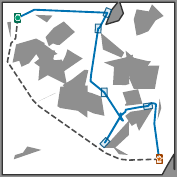}}\hfill
  \subfloat[]{\includegraphics[width=\dimexpr\textwidth/10-1pt\relax]{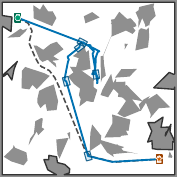}}\hfill
  \subfloat[]{\includegraphics[width=\dimexpr\textwidth/10-1pt\relax]{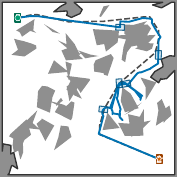}}\hfill
  \subfloat[]{\includegraphics[width=\dimexpr\textwidth/10-1pt\relax]{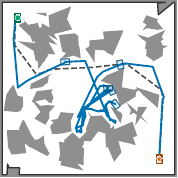}}\hfill
  \subfloat[]{\includegraphics[width=\dimexpr\textwidth/10-1pt\relax]{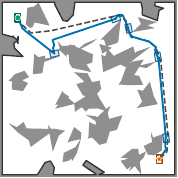}}\hfill
  \subfloat[]{\includegraphics[width=\dimexpr\textwidth/10-1pt\relax]{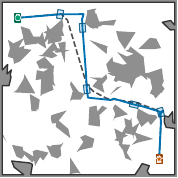}}\hfill
  \subfloat[]{\includegraphics[width=\dimexpr\textwidth/10-1pt\relax]{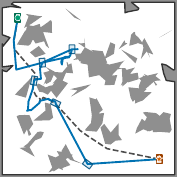}}
  \caption{Ten hard cases at $N=360$: executed \dRVG routes (blue) and full-map RVG routes (gray dashed). Footprints show robot orientation; green and orange mark start and goal.}
  \label{fig:routes}
  \setcounter{subfigure}{0}
  \subfloat[\dRVG]{\includegraphics[width=\dimexpr\textwidth/7-1pt\relax]{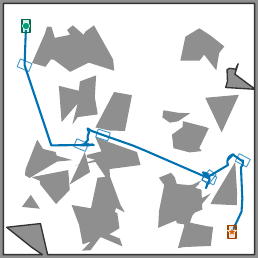}}\hfill
  \subfloat[$g+h$]{\includegraphics[width=\dimexpr\textwidth/7-1pt\relax]{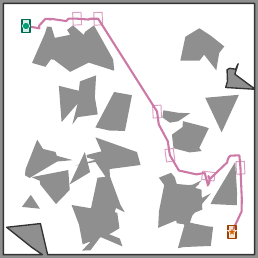}}\hfill
  \subfloat[A*]{\includegraphics[width=\dimexpr\textwidth/7-1pt\relax]{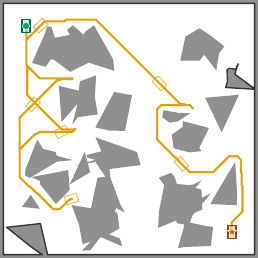}}\hfill
  \subfloat[D* Lite]{\includegraphics[width=\dimexpr\textwidth/7-1pt\relax]{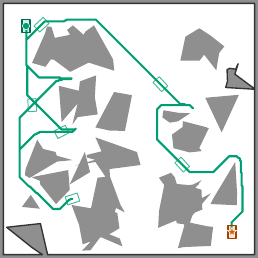}}\hfill
  \subfloat[$\mathrm{RRT}^{X}$]{\includegraphics[width=\dimexpr\textwidth/7-1pt\relax]{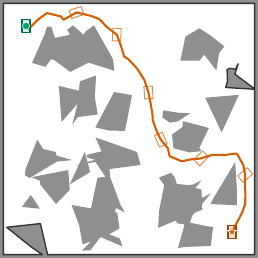}}\hfill
  \subfloat[EIT*]{\includegraphics[width=\dimexpr\textwidth/7-1pt\relax]{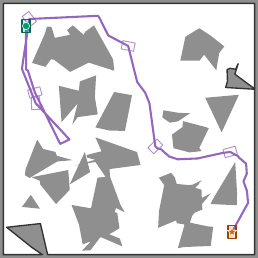}}\hfill
  \subfloat[Static RVG]{\includegraphics[width=\dimexpr\textwidth/7-1pt\relax]{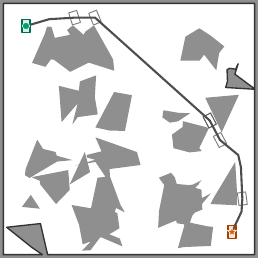}}
  \caption{Compared planners on one case at $N=72$. The occupancy-grid cell size is $\Delta x=0.10$; EIT* receives 100\,ms per fresh query. Colored footprints show execution; static RVG has the complete map.}
  \label{fig:method-routes}
  \vspace{-2mm}
\end{figure*}

\begin{figure*}[!t]
  \centering
  {\setlength{\tabcolsep}{1pt}\renewcommand{\arraystretch}{0}
  \begin{tabular}{@{}c c c c c@{}}
    & \small Success (\%) & \small Planning time (ms) & \small Cost $J/J_{\mathrm{RVG}}$ & \small Nodes / states \\[2pt]
    \raisebox{-.5\height}{\rotatebox{90}{\small $\Delta x=0.20$}} & \raisebox{-.5\height}{\includegraphics[width=0.235\textwidth,clip]{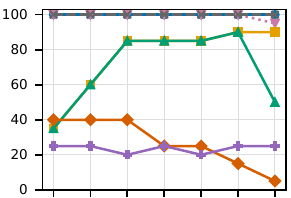}} & \raisebox{-.5\height}{\includegraphics[width=0.235\textwidth,clip]{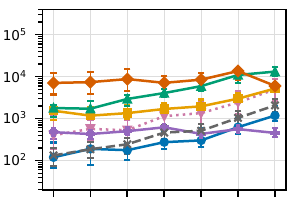}} & \raisebox{-.5\height}{\includegraphics[width=0.235\textwidth,clip]{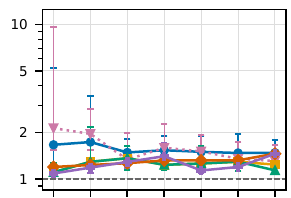}} & \raisebox{-.5\height}{\includegraphics[width=0.235\textwidth,clip]{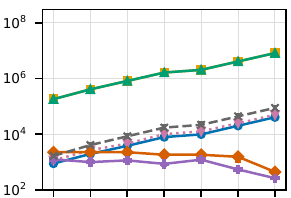}} \\[2pt]
    \raisebox{-.5\height}{\rotatebox{90}{\small $\Delta x=0.10$}} & \raisebox{-.5\height}{\includegraphics[width=0.235\textwidth,clip]{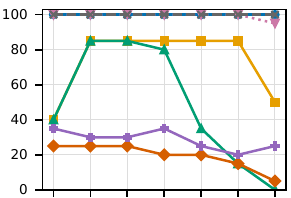}} & \raisebox{-.5\height}{\includegraphics[width=0.235\textwidth,clip]{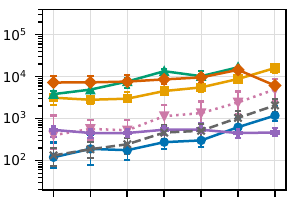}} & \raisebox{-.5\height}{\includegraphics[width=0.235\textwidth,clip]{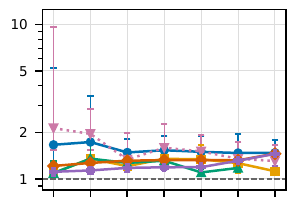}} & \raisebox{-.5\height}{\includegraphics[width=0.235\textwidth,clip]{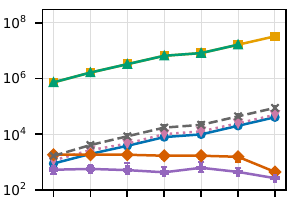}} \\[2pt]
    \raisebox{-.5\height}{\rotatebox{90}{\small $\Delta x=0.05$}} & \raisebox{-.5\height}{\includegraphics[width=0.235\textwidth,clip]{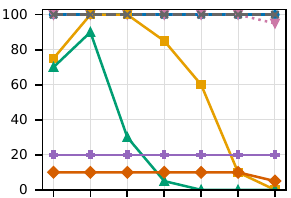}} & \raisebox{-.5\height}{\includegraphics[width=0.235\textwidth,clip]{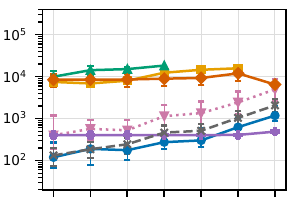}} & \raisebox{-.5\height}{\includegraphics[width=0.235\textwidth,clip]{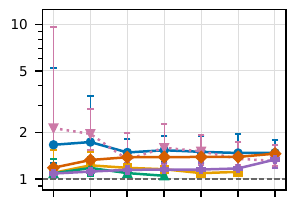}} & \raisebox{-.5\height}{\includegraphics[width=0.235\textwidth,clip]{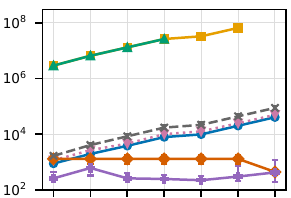}} \\[2pt]
    & \includegraphics[width=0.235\textwidth,clip]{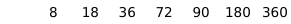} & \includegraphics[width=0.235\textwidth,clip]{figures/comparison_orientation_ticks.pdf} & \includegraphics[width=0.235\textwidth,clip]{figures/comparison_orientation_ticks.pdf} & \includegraphics[width=0.235\textwidth,clip]{figures/comparison_orientation_ticks.pdf} \\
    & \small Orientation bins, $N$ & \multicolumn{3}{c}{\raisebox{-.35\height}{\includegraphics[width=0.70\textwidth]{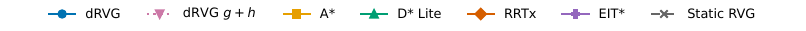}}}
  \end{tabular}}
  \caption{Comparison on 20 hard maps with a 20 s planning budget. Rows correspond to occupancy-grid cell sizes $\Delta x$; columns show success, planning time, cost relative to full-map RVG, and the number of graph nodes or grid states. Success rates use all runs. For runtime, cost, and the counts in the last column, markers show medians over successful runs and bars span the 25th--75th percentiles; missing markers indicate no success. Here $\Delta x$ denotes the occupancy-grid cell size; it also sets collision-check spacing for continuous planners.}
  \label{fig:baseline}
\end{figure*}

\begin{figure}[t]
  \centering
  \begin{minipage}[t]{0.32\columnwidth}
    \centering\footnotesize Success (\%)\par\smallskip
    \includegraphics[trim=11bp 0 0 0,clip,width=\linewidth]{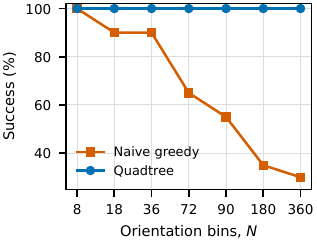}
  \end{minipage}\hfill
  \begin{minipage}[t]{0.32\columnwidth}
    \centering\footnotesize Runtime (s)\par\smallskip
    \includegraphics[trim=11bp 0 0 0,clip,width=\linewidth]{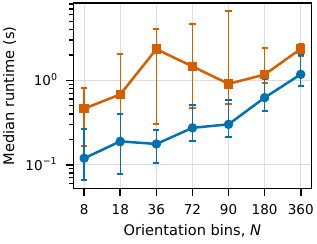}
  \end{minipage}\hfill
  \begin{minipage}[t]{0.32\columnwidth}
    \centering\footnotesize Cost $J/J_{\mathrm{RVG}}$\par\smallskip
    \includegraphics[trim=11bp 0 0 0,clip,width=\linewidth]{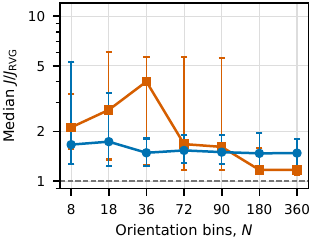}
  \end{minipage}
  \caption{Quadtree \dRVG and naive greedy \dRVG without quadtree scheduling. Success rates use all runs. Runtime and normalized cost summarize successful runs with medians and interquartile ranges.}
  \label{fig:ablation}
  \vspace{-1mm}
\end{figure}

We evaluate scaling with angular resolution on difficult maps, the benefit of spatial scheduling, and physical execution of the online planning loop.

\subsection{Protocol and Metrics}
The 20 test maps each span $30\times30$ workspace units, and the robot is a rectangle measuring $1\times1.5$ units. Let $\Delta x$ denote the cell size of the planar $\mathbb{R}^2$ occupancy grid used by repeated A*~\cite{hart1968astar} and D* Lite~\cite{koenig2002dstarlite}. We test $\Delta x\in\{0.05,0.1,0.2\}$ and $N\in\{8,18,36,72,90,180,360\}$ orientation layers. Static RVG~\cite{zhang2025rvg} solves every tested map--resolution case.

The main comparison contains 2,100 runs. Static RVG, default \dRVG, and \dRVG with the $g+h$ selector each have 140 runs: one for every map and angular resolution. Repeated A*, D* Lite, $\mathrm{RRT}^{X}$~\cite{otte2016rrtx}, and replanning from scratch with EIT*~\cite{strub2022eitstar,sabbadini2026replanning} each have 420 runs: every map and angular resolution is tested at each of the three $\Delta x$ values.

\dRVG scans at the start and at each reached temporary goal. The online experiments use the footprint-scan model; path costs and sensing-goal scores use $\alpha=1$ and $\beta=0.1$. \dRVG splits leaves with over 32 sites if each child is at least one robot diameter wide.

Each run has a 20 s cumulative planning-time budget. Success requires goal arrival without exhausting this budget and complete-map collision validation. EIT* receives at most 100\,ms per planning query, following the setting with the highest success rate in Sabbadini et al.'s query-budget comparison~\cite{sabbadini2026replanning}. The normalized path cost is $J/J_{\rm RVG}$, where $J_{\rm RVG}$ is the shortest-path cost in the full-map RVG for the same map and $N$. We compute success rates over all runs. Figure~\ref{fig:baseline} reports medians of planning time, normalized executed cost, and graph-node or grid-state counts over each method's successful runs; error bars span the 25th--75th percentiles. These subsets may differ across methods. For cases solved by both sensing selectors, we report the geometric mean and median of their runtime ratios. Experiments run on an Intel Core i9-14900KF. Repeated A* and D* Lite use 24 threads for grid construction; all methods use one for graph search. Sampling-based planners use seed 1.

\emph{Baseline planning and execution.}
Repeated A* and D* Lite use conservative $SE(2)$ lattices covering each position cell and orientation interval. They accumulate visible obstacle-boundary segments and treat unseen space optimistically. $\mathrm{RRT}^{X}$ and EIT* use continuous $SE(2)$ and receive complete obstacle polygons when any part becomes visible. For these continuous planners, $\Delta x$ controls collision checking. Steps are $\Delta x/4$ and $2\pi/(4N)$ during planning, and $\Delta x/40$ and $2\pi/(40N)$ for final validation.

All online planners follow a scan--plan--move--scan cycle; their chosen motions produce different observation sequences. Each planner incorporates the latest observation, computes a route, and executes a portion of it before scanning again at the reached pose. If no path is returned for the next move, the run fails: the robot cannot advance to obtain a new observation. We use the OMPL implementation of EIT*~\cite{sucan2012ompl}, rejecting approximate solutions and applying no smoothing.

\subsection{Planner Comparison}
Default \dRVG succeeds in all 140 map--resolution cases; its optional $g+h$ selector succeeds in 139/140. Figure~\ref{fig:baseline} shows the results separately for each occupancy-grid cell size $\Delta x$. At $N=360$, default \dRVG's median planning time is 1.18 s. Using each baseline's lowest median runtime across the three grid cell sizes, repeated A*, D* Lite, and $\mathrm{RRT}^{X}$ take $4.40\times$, $11.15\times$, and $5.09\times$ as long, respectively. Among successful runs, EIT* has lower median runtimes of 0.49, 0.47, and 0.45 s for $\Delta x=0.05$, $0.1$, and $0.2$, but it succeeds in only 4/20, 5/20, and 5/20 cases, respectively. Across all angular resolutions and grid cell sizes, EIT* succeeds in 101 of the 420 runs; 10 additional reported solutions fail the final collision check.

The fourth column of Figure~\ref{fig:baseline} reports the number of graph nodes or occupancy-grid states. At $N=360$, the default \dRVG and $g+h$ graphs contain medians of 40,358 and 49,115 vertices, versus 87,370 for static RVG. The lattice planners allocate 129.6, 32.4, and 8.1 million states at $\Delta x=0.05$, $0.1$, and $0.2$, respectively. At this angular resolution, the lattices contain approximately 201--3,211 times as many states as the median default \dRVG roadmap, illustrating the compactness of its geometric representation. The one successful $\mathrm{RRT}^{X}$ tree per cell size contains 440 nodes. For EIT*, we count nodes in each query graph and retain the largest count during a run. The medians over successful runs are 422.5, 262, and 262 nodes for $\Delta x=0.05$, $0.1$, and $0.2$, respectively.

The tested online runs incur higher executed cost than full-map planning. At $N=360$, median $J/J_{\rm RVG}$ is 1.477 for default \dRVG and 1.306 for $g+h$; baseline medians across the grid cell sizes range from 1.122 to 1.457. Cost medians may cover different maps because successful subsets vary.

Finer grids increase spatial resolution but can reduce success under the planning-time budget. At $N=360$, repeated A* solves no case at $\Delta x=0.05$; it solves 10/20 at $0.1$ and 18/20 at $0.2$, with median runtimes of 16.02 and 5.19 s over successful runs. D* Lite solves no case at $0.05$ or $0.1$ and 10/20 at $0.2$, with a median runtime of 13.15 s over successful runs. $\mathrm{RRT}^{X}$ solves 1/20 at every collision-check resolution, with runtimes from 6.00 to 6.46 s for those successful runs. \dRVG has no translation-cell parameter.

We also vary the EIT* query budget at $\Delta x=0.1$ and $N=72$. Budgets of 10, 50, and 100\,ms yield 0/20, 3/20, and 7/20 successful runs, respectively.

\subsection{Naive Selector Reference}
Figure~\ref{fig:ablation} compares 140 runs per selector using the same footprint-scan model, maps, and angular resolutions. Both selectors exclude previously scanned positions, regardless of orientation. Naive greedy minimizes Eq.~\eqref{eq:selector-score} over all reachable configurations at unscanned sites; quadtree selection additionally defers served regions.

Quadtree \dRVG solves 140/140 cases, whereas naive greedy solves 93/140; all 47 naive-greedy failures are timeouts at the 20 s planning budget. Across jointly solved cases, greedy requires $5.51\times$ the runtime and $2.75\times$ as many scan--plan--move iterations in geometric mean; the paired median runtime ratio is $3.94\times$. At $N=360$, quadtree solves 20/20 versus greedy's 6/20, with a $3.84\times$ paired geometric-mean runtime ratio. Considering each method's successful runs separately, median runtimes are 1.18 s for quadtree and 2.36 s for greedy, and median $J/J_{\rm RVG}$ values are 1.477 and 1.170, respectively. At $N=360$, the greedy cost median uses six successful runs, whereas the quadtree median uses all 20.

\subsection{Real-Robot Demonstration}\label{sec:real-robot}

We use the microMVP platform~\cite{YuHanTanRus17ICRA} to demonstrate \dRVG. At each planning iteration, we compute footprint-scan observations and visible obstacle geometry from an overhead-camera image. We use $N=36$ orientations with $\alpha=1.0$ and $\beta=0.1$.

We demonstrate navigation in six arrangements of polygonal obstacles. In each recorded run, the robot reaches its requested goal pose through successive translation and rotation motions. Figure~\ref{fig:real-robot} shows a representative sequence: \dRVG visits temporary sensing goals, incorporates newly observed space, and reorients the robot around obstacles until a route to the final goal becomes available. Parts of the workspace remain unobserved at goal arrival. The supplementary video presents all six runs at their original speed.

\section{Conclusion}\label{sec:conclusion}
\dRVG merges local RVGs and uses quadtree sensing to navigate polygonal robots among unknown static obstacles. Under the stated assumptions and without a time limit, center-scan \dRVG reaches any goal connected in full-map RVG at the same angular resolution. The analysis does not establish this guarantee for arbitrary polygonal robots or footprint scans. In footprint-scan experiments, default \dRVG solved all 140 map--resolution cases within the planning budget; the ablation supports quadtree scheduling over naive greedy selection. Six microMVP demonstrations use overhead-camera observations.

Integrating \dRVG with simultaneous localization and mapping (SLAM) would update the roadmap as pose and map estimates change. Frontier selection could then balance goal progress, information gain, and localization uncertainty. Loop closures would require revising roadmap geometry and revalidating affected connections. Other directions include adapting spatial scheduling to finite-range sensing and assessing how map uncertainty affects collision margins. These extensions would support navigation with onboard sensing.

{\footnotesize
\bibliographystyle{formatting/IEEEtran}
\bibliography{bib/jingjin}
}

\end{document}